\documentclass{article}

\usepackage{microtype}
\usepackage{graphicx}
\usepackage{subcaption}
\usepackage{booktabs} 
\usepackage{nicefrac}       

\usepackage{hyperref}

\usepackage{amsmath}
\usepackage{amssymb}
\usepackage{mathtools}
\usepackage{amsthm}

\usepackage[capitalize,noabbrev]{cleveref}

\theoremstyle{plain}
\newtheorem{theorem}{Theorem}[section]
\newtheorem{proposition}[theorem]{Proposition}
\newtheorem{lemma}[theorem]{Lemma}

\theoremstyle{definition}
\newtheorem{definition}[theorem]{Definition}
\newtheorem{assumption}[theorem]{Assumption}
\theoremstyle{remark}
\newtheorem{remark}[theorem]{Remark}

\usepackage[textsize=tiny]{todonotes}

\begin{document}

  \title{Causal Inference on Stopped Random Walks in Online Advertising}

    \author{Jia Yuan Yu}
    \date{\today}

\maketitle

\begin{abstract}
We consider a causal inference problem frequently encountered in online advertising systems, where a publisher (e.g., Instagram, TikTok) interacts repeatedly with human users and advertisers by sporadically displaying to each user an advertisement selected through an auction.
Each treatment corresponds to a parameter value of the advertising mechanism (e.g., auction reserve-price), and we want to estimate through experiments the corresponding long-term treatment effect (e.g., annual advertising revenue).
In our setting, the treatment affects not only the instantaneous revenue from showing an ad, but also changes each user's interaction-trajectory, and each advertiser's bidding policy---as the latter is constrained by a finite budget.
In particular, each a treatment may even affect the size of the population, since users interact longer with a tolerable advertising mechanism.
We drop the classical i.i.d. assumption and model the experiment measurements (e.g., advertising revenue) as a stopped random walk, and use a budget-splitting experimental design, the Anscombe Theorem, a Wald-like equation, and a Central Limit Theorem to construct confidence intervals for the long-term treatment effect.
\end{abstract}




\section{Introduction}


Online advertising platforms---or publishers for short---often encounter the problem of selecting one out of several competing changes to deploy to their users.
This selection is usually done by answering a causal inference question such as: What is the annualized lift in revenue if we change this parameter of the mechanism that selects which ads to display to the user?
Given an ad auction mechanism, should we add a small fee to each winner's payment or change the reserve price? 

\subsection{Markovian dynamics}

The classical approach to answer inference questions is to run an A/B test---or experiment, and use the Central Limit Theorem (CLT) to generalize an empirical average over a few measurements to performance guarantees about the system as a whole.
The classical approach works well on isolated systems (e.g., individual drug treatments on patients), but online advertising is a large interconnected system.
In such systems, inference is a hard because the mere decision to take measurements of the treatment effect changes what we want to measure. To illustrate, imagine that you are a police agent being tasked to measure the congestion on different roads, while driving your patrol vehicle. You can measure the congestion by driving around, but your mere presence changes the traffic patterns in a broad area---with an effect that sometimes persists even after you are gone. Moreover, once you take one road, you only observe that road, you can't observe the traffic on other roads at the same time.

In online advertising, when we change the method of allocating ads to impression opportunities, we show different ads and thereby change the trajectory of the user, resulting in different clicks, purchases, and numbers of impression opportunities. Moreover, we change the trajectory of the advertisers' budgets, with may result in a willingness to pay higher amounts for an impression when there is a leftover budget near a deadline.

To tackle inference in such complex systems, we model the trilateral interactions between publisher, users, and advertisers as a Markov chain. For instance, we model the user's page-transitions on the publisher's webpage as a Markov chain: each click on a hyperlink is a transition from one page to another. Each advertiser has a budget that updates as another Markov chain at each impression opportunity auction.

\subsection{Random treatment-dependent population}

A new difficulty arises when we want to infer a long-term business metric, such as the publisher's revenue over a time-scale much larger than that of ad impressions.
This long-term metric clearly depends on both the number of ad impressions delivered and the payment received for each impression.
We model it as a partial sum associated with a Markov chain with a stopping time, i.e., a stopped random walk. 
In turn, we employ a budget-splitting experimental design and infer the treatment effect with another stopped random walk.
As a result, we generalize the deterministic population size of classical causal inference to a random treatment-dependent population.

\subsection{Application to ad auctions}

As an illustrative example, consider a publisher has a set of webpages or pages, and each page contains a placement (slot) dedicated to displaying ads. Whenever an user visits a page, an impression opportunity is generated: each advertiser has an opportunity to impress the user with its ad. This opportunity is allocated to an advertiser through a auction (e.g., second-price or first-price auction) with a publisher-set reserve price.
Advertisers submit bids in this auction.
Each bid is a function of the page-context (what appears on the page, etc.), user-context (search keywords, etc.), placement-context (click through rate, etc.), and advertiser-context (budget remaining, deadline, etc.). 

There is an opportunity to optimize the reserve price of the auction: when it is too high, many opportunities are unsold, if it is too low, it does not affect the advertisers' payments.
Although the optimal reserve price is well-known in the setting of a single isolated auction \cite{myerson}, our setting is much more complex. Various strategies---or treatments---have been proposed under different assumptions. One way to select the best strategy is to run experiments, and generalize measurements from the experiments to the treatment effects of deploying these treatments in the long-term.


Suppose that we are interested in the publisher's advertising revenue.
To define this effect, we first need to consider two trajectories or potential outcomes \cite{po} corresponding to two treatments. 
For discussion sake, consider two treatments corresponding to two values for the reserve price in the auctions that allocate ads to impression opportunities.
For example, suppose that the first trajectory corresponds to an user session with reserve price 5 USD, and the second trajectory is a session for the same user with a reserve price of 1 USD. 
We can only observe one of these two trajectories, say the first one.
In the observed trajectory, most auctions result in unsold impression opportunities, advertisers pay more for each impression, and the user purchases a toy.
In the other trajectory, most ad impressions are delivered, possibly resulting is a longer session where the user purchases a laptop.
Here, the effect of interest\footnote{In other applications, we may be interested in the joint effect on advertising and retail revenue.} is the difference between the expected advertising revenue that the publisher accumulates over these two trajectories.

In the following sections, we present the model of the interactions between participants in online advertising (Section~\ref{sec:model}), describe the inference problem (Section~\ref{sec:in}), and the corresponding performance guarantee (Section~\ref{sec:main}).
We make three main contributions.
\begin{enumerate}
    \item On the modeling front, we introduce a new notion of treatment effect with a random treatment-dependent population size, and a stopped random walk model for the complex interactions between users, advertisers, and the publisher.
    \item On the analytical front, we introduce a permutation trick to greatly reduce the dimension of the state space at the expense of a mild assumption on the relative time scales of user-session and advertiser campaign duration.
    \item We combine a Markovian CLT with a Wald-like equation and the Anscombe Theorem to derive confidence intervals in spite of the random population size. 
    This provide in turn guidance on the necessary duration for experiments.
\end{enumerate}


\section{Model}\label{sec:model}

In this section, we model the user-publisher-advertiser interactions in online advertising, with the goal of capturing two main observations:
\begin{enumerate}
    \item Each user's session, i.e., sequence of pages visited between login and logout, is a trajectory that is approximately Markovian;
    \item Each advertiser's bidding policy is a function of the current auction and the remaining budget, which itself is a function of the trajectory of past auctions.
\end{enumerate}

To make things concrete, consider first the simplest case of a sequence of impression opportunities $1,2,\ldots$ in a setting where a single user navigates through pages on the publisher's webpage, and where there is a single advertiser. 
Throughout the paper, we will refer to these impression opportunities as ``\emph{opportunities},'' ``\emph{impressions},'' or ``\emph{pages}'' depending on the context.
As the user navigates through pages of the publisher, we obtain the following sequence of pairs of opportunities and remaining ad budgets: $(X_1,\Omega_1),\ldots, (X_n,\Omega_n), \ldots$ This sequence is Markovian if we assume that the user goes from one page $X_n$ to the next $X_{n+1}$ by clicking on hyperlinks and ads on each page\footnote{A similar Markovian model of user behavior on the Web is the basis of Google's Pagerank algorithm \cite{pagerank}; where webpages are states and transition probabilities are assigned to hyperlinks.}. The advertiser's remaining budget is decremented by payments to the publisher for each rendering of or each click on an ad, and incremented by budget replenishment over time.


To generalize this model to multiple users and multiple advertisers competing through repeated auctions,
consider a setting with a set of users $[d] = \{1,\ldots,d\}$ and $m$ advertisers, and let $n=1,2,\ldots$ denote the sequence of impression opportunities or page-transitions.
For each $n$, we let $I_n \in [d]$ denote the user who transitions from one page to another, which creates the opportunity $n$.
We let $X^i_n \in \mathcal X$ denote the page that is displayed to user $i$ at the impression opportunity $n$. 
In general, the state space $\mathcal X$ for the ``page'' also encodes characteristics of the user-session that allow advertisers to bid accordingly.
Opportunity $n$ coincides with a page-transition for user $I_n$ from $X^{I_n}_{n-1}$ to $X^{I_n}_n$; there is no transition for other users: $X^j_n = X^j_{n-1}$.
For each $n$, we define the profiles of pages for all users as $\mathbf X_n = (X_n^1, \ldots, X_n^d) \in \mathcal X^d$, and the profile of budget remaining for all advertisers as $\Omega_n \in \mathbb R^m$.
The sequence of interactions can be described by the Markov chain
\begin{align}\label{eq:mc}
    (\mathbf X_1, \Omega_1, I_1), \ldots, (\mathbf X_n, \Omega_n, I_n), \ldots,
\end{align}
where the state-space is now much larger than the single-user single-advertiser case.

\subsection{Permutation trick to reduce dimensionality}

When the number of users $d$ is large, the Markov chain \eqref{eq:mc} has a huge state space containing $\mathcal X^d$ that cannot be handled by existing methods to analyze convergence to an ergodic limit. 
We introduce in this section a permutation trick that approximates the Markov chain \eqref{eq:mc} with another Markov chain that encodes essentially the same information, but has a vastly smaller state space.

The permutation trick is motivated by the following observation: at a store's checkout conveyor belt, whether we group items by type, or even interleave items from different users, effect on sales and cashflow is localized.
Likewise, whether the publisher serves ads sequentially to multiple concurrent users or sequentially to one user after another, the trajectory of pages visited and the trajectory of remaining-budgets are approximately the same.
This is true as long as the time- and budget-scales of user-sessions are much smaller than the scales of entire ad campaigns.


Informally, the permutation groups together all page-transitions corresponding to the same user-session (sharing the same value $I_n$). This combination effectively allows us to turn a Markov chain where page-transitions from concurrent sessions are interleaved into another Markov chain where only entire user-sessions are interleaved.


\begin{figure*}
     \centering
\begin{tabular}{l*{6}{c}r}
Impression $n$                   & 1 & 2 & 3 & 4 & 5  & 6 & \ldots \\
Active user index ($I_n$)      & 1 & 2 & 1 & 3 & 2 & 1 & \ldots \\
Page ($X^{I_n}_n$) & $X^1_1$ & $X^2_2$ & $X^1_3$ & $X^3_4$ & $X^2_5$ & $X^1_6$ & \ldots \\
Permuted user index $\tilde I_{n}$          & 1 & 1 & 1 & 2 & 2 & 3 & \ldots \\
Permuted page $\tilde X_n$ & $X^1_1$ & $X^1_3$ & $X^1_6$ & $X^2_2$ & $X^2_5$ & $X^3_4$ & \ldots \\
\end{tabular}
        \caption{Permutation $\sigma$ with $T=6$ that groups page-transitions together by user-session.}
        \label{fig:perm}
\end{figure*}

We construct the approximate Markov chain with the following infinite permutation.
For a fixed $T$, let $\sigma$ denote the permutation function that iterates over $i=1,2,\ldots$, creates a fixed window of $T$ impressions $U_i = (i, \ldots, i+T-1)$, and permutes the impressions $U_i$ such that they are ``grouped together'' by active user.
Formally, let $I_1, I_2, \ldots$ denote the original sequence of active users, and let 
\begin{align*}
    \tilde I_1, \tilde I_2, \ldots = I_{\sigma^{-1}(1)}, I_{\sigma^{-1}(2)}, \ldots
\end{align*}
denote the permuted sequence of active users.
The permuted sequence satisfies the following properties for all indices $k,\ell,m \in \{k,\ldots,m\}$:
\begin{enumerate}
    \item (grouping by active user) if $|k-m|\leq T$ and $I_{k} = I_{m}$, then $I_\ell = I_k$ for all $\ell \in \{k,\ldots,m\}$;
\item (order-preserving) if $k < m$, then $\sigma(k) < \sigma(m)$.
\end{enumerate}
The window length $T$ ensures that we only group pages traversed during a single user-session.
Figure~\ref{fig:perm} illustrates the permutation $\sigma$ with an example\footnote{This permutation is connected to the problem of partial sorting with a sliding window}.

The permutation $\sigma$ encodes essentially the same information as the Markov chain $(\mathbf X_n, \Omega_n, I_n)$, but with impressions are grouped sequentially by user-session.
The permuted Markov chain over the much smaller state space $\mathcal X \times \mathbb R^m \times [d]$ is denoted by
\begin{align*}
    \Phi_n \triangleq (\tilde X_n, \tilde \Omega_n, \tilde I_n) = (X_{\sigma^{-1}(n)}^{I_{\sigma^{-1}(n)}}, \tilde \Omega_n, I_{\sigma^{-1}(n)}).
\end{align*}
At every page-transition $n$, if the active user $\tilde I_{n-1}$ transitions to a new user, then the next page $\tilde X_n$ is the first page of the new user's session, otherwise, the page-transition from $\tilde X_{n-1}$ to $\tilde X_n$ follows the same probability law as in the Markov chain \eqref{eq:mc}.
Whereas $\tilde X_n$ and $\tilde I_n$ are constructed by applying the same permutation on $X_n$ and $I_n$, the sequence $\tilde \Omega_n$ is not a permutation of $\Omega_n$: it denotes the dynamics of the remaining-budget for all advertisers if they compete on the permuted sequence of impression opportunities $\tilde I_n$.

The permutation trick gives us a powerful way to perform experiments on the actual advertising system, while analyzing convergence on a much smaller system.
To control the incurred approximation error, we assume throughout this paper that the length of ad campaigns is large compared to the length of  sessions and the number of concurrent users, and the ad campaign budget is large compared to the bid on each individual impression.

\begin{assumption}[Scale of ad campaigns]\label{as:session}
The maximum number of impressions in an user-session is $L$. At every given time instant, the maximum number of concurrent users-sessions is $S$, so that all pages corresponding to the same user-session have indices at most $LS$ apart in the Markov chain $\{\mathbf X_n\}$ of \eqref{eq:mc}.
The maximum payment by an advertiser on an impression is $B$.
Informally, we assume that every ad campaign has a total budget that is much larger than $B$ and a duration much larger than $L$ and $S$.
\end{assumption}

\subsection{State-transition dynamics for online advertising}\label{sec:mc}

In this section, we describe the state-transition dynamics that generate the Markov chain $\Phi_n$. In our online advertising setting, advertisers and the publisher interact through profiles of bids $Y_1,Y_2,\ldots$ and auction reserve prices $p_1,p_2,\ldots$.
On the one hand, each advertiser uses a sequence of bid to control the page trajectory $\tilde X_n$ toward its target page to sell its product, while paying as little as possible.
On the other hand, the publisher uses a sequence of reserve prices to control both the page-trajectory $\tilde X_n$ and advertiser budget-trajectory $\tilde \Omega_n$ so as to maximize long-term publisher revenue.
For simplicity, we assume that the publisher and every advertiser follow stationary strategies, which may be learned from historical interactions.

Specifically, we assume that the publisher takes the current page $\tilde X_n$ and generates a reserve price
\begin{align}\label{eq:p}
    p_n = p(\tilde X_n), \quad\mbox{for all }n=1,2,\ldots,
\end{align}
with a fixed function $p:\mathcal X\to \mathbb R$. In practice, the publisher learns a sequence of functions $p_1,p_2,\ldots$ from a history of interactions (bid profiles, revenues, etc.), but we do not model the learning problem here.
More generally, the publisher may set a personalized reserve price for each advertiser, but we consider for simplicity a single reserve price for all bidders.

The $m$ advertisers take the page $\tilde X_n$, the reserve price $p_n$, and the profile of remaining budgets $\tilde \Omega_n$, and respond with a bid profile
\begin{align}\label{eq:y}
    Y_n = Y(\tilde X_n,p_n,\tilde \Omega_n), \quad \mbox{for all }n=1,2,\ldots,
\end{align}
with a fixed function $Y: \mathcal X \times \mathbb R \times \mathbb R^d \to \mathbb R^m$.
In practice, each advertiser entrusts a campaign budget to a proxy bidder that participates in the second-price on its behalf.


In practice, the page $\tilde X_n$ is not simply a page, but a context composed of an array of categorical and numerical variables, containing notably information such as how many times each user has viewed each ad, histories of user-advertisement interactions, content surrounding the impression opportunity, search keywords, etc.

Let $\xi_n$ denote a sequence of i.i.d. random variables that models the uncertainty in the decisions of the users: whether a session ends, which hyperlink the user clicks on, etc.
We assume that the page $\tilde X_{n+1}$ of the next impression is a function of the current page $\tilde X_n$, the reserve price $p_n$, the bid profile $Y_n$, the current active user $\tilde I_n$, and the auxiliary random variable $\xi_n$: 
\begin{align}\label{eq:x}
    \tilde X_{n+1} &= h(\tilde X_n, p_n, Y_n, \tilde I_n, \xi_n)\nonumber\\
    &= h\Big(\tilde X_n, p(\tilde X_n), Y\big(\tilde X_n, p(\tilde X_n), \tilde \Omega_n\big), \tilde I_n, \xi_n\Big),
\end{align}
for $n=1,2,\ldots$,
where the last equality follows from \eqref{eq:p} and \eqref{eq:y}.

Let $\zeta_n$ denote a sequence of i.i.d. random variables that models the uncertainty in the budget replenishment decisions of advertisers.
We assume that the remaining-budget profile $\tilde \Omega_n$ is updated according to the payment made by the winning advertiser (if any) of each auction and the replenishment decision of each advertiser.
Hence, $\tilde \Omega_{n+1}$ is a function of the current profile $\tilde \Omega_n$, the payment for the current impression, and the auxiliary random variable $\zeta_n$:
\begin{align}\label{eq:omega}
    \tilde \Omega_{n+1} &= g(\tilde \Omega_n, p_n, Y_n, \zeta_n)\nonumber\\
    &= g\Big(\tilde \Omega_n, p(\tilde X_n), Y\big(\tilde X_n, p(\tilde X_n), \tilde \Omega_n\big), \zeta_n\Big),
\end{align}
for $n=1,2,\ldots$
This sequence is private to the advertisers and not observed by the publisher.


\begin{remark}[Two time scales]
The process $\tilde \Omega_n$ evolves on a much slower time scale compared to $\tilde X_n$ because each impression costs a tiny fraction of the overall budget.
\end{remark}

\subsection{Publisher revenue}

This section describes the publisher's auction revenue that forms the basis for the target metric for causal inference.
For each impression opportunity $n$,
the winner among $Y_n$ of the second-price auction with (non-personalized) reserve price $p_n$ is allocated the opportunity and
pays the publisher a fee of $r(p_n,Y_n)$. 
Observe that $r: \mathbb R \times \mathbb R^m$ is a non-linear (even discontinuous) function of reserve price:
\begin{align*}
    r(p_n, Y_n) = \mathrm{smax}(p_n, Y_n) \mathbb I_{[\max(Y_n) \geq p_n]},
\end{align*}
where $\mathrm{smax}(Y_n,p_n)$ denote the second-highest value among bids in $Y_n$ and reserve price $p_n$, and $\mathbb I$ denotes the indicator function. Figure~\ref{fig:reserve-rev} illustrates the auction revenue as a function of the reserve price in the presence of two bids of 1 and 2 USD.

\begin{figure}
         \centering
         \includegraphics[width=\textwidth/2]{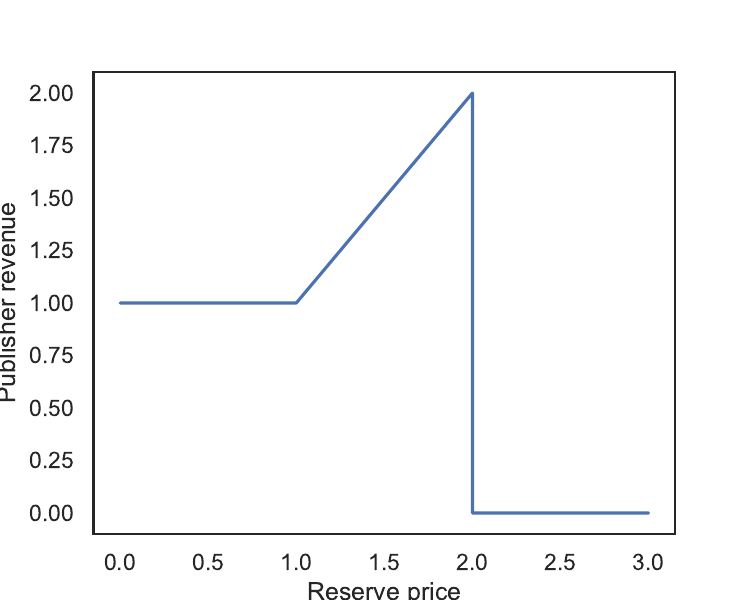}
         \caption{Auctioneer revenue $r(\cdot,Y)$ for a single auction with fixed bid profile $(1,2)$.}
        \label{fig:reserve-rev}
\end{figure}




\section{Inference}\label{sec:in}

In this section, we evaluate the performance of two reserve price treatments by using a variant of randomized experiments.
First, we define the notion of long-term treatment effect  in a new setting where both the number of interactions and the revenue per interaction are affected by the treatment (e.g., total revenue over $k$ days).
Our main result is a confidence bound for a method to infer the treatment effect from a randomized experiment \cite{linkedin}.

Although this work extends the causal inference literature, we make a conscious effort to appeal to a broader audience by using language from statistics instead of more technical term like "units," SUTVA, etc.

\subsection{Expected long-term treatment effect}

In this section, we define the expected long-term treatment effect that we want to infer through experiments.
We start from a setting similar to the potential outcomes model of \cite{po}.
Consider a population composed of impression opportunities and two treatments corresponding to two publisher strategies for reserve prices. For simplicity and without loss of generality, we assume that these strategies are static reserve prices $v \in\mathbb R$ and $w \in\mathbb R$.
The publisher is interested in optimizing the expected ad revenue accumulated over a long period of time, say $k$ days. In a major departure from classical causal inference, the number of impressions in $k$ days is random and depends on the publisher's reserve-price strategy.
We model this long-term revenue as a stopped random walk.

\begin{definition}[Markov random walk, stopped random walk]\label{def:Phi}
We let $\Phi(v)$ denote the Markov chain $\Phi$ induced by the constant reserve price function $p(X) = v$ for all $X \in \mathcal X$.
The partial average revenue and partial sum revenue are denoted by
\begin{align}\label{eq:bar-r}
S_n(v, \Phi(v)) &= \sum_{i=1}^n r(v, Y(\tilde X_i, v, \tilde \Omega_i)), \nonumber\\
    \bar r_n(v, \Phi(v)) &= \frac{1}{n} \sum_{i=1}^n r(v, Y(\tilde X_i, v, \tilde \Omega_i)).
\end{align}
The sequence $\{(\Phi_n, S_n)\}$ is a \emph{Markov random walk} \cite{fuhlai}. The additive component $\{S_n\}$ is a renewal process \cite{anscombe60} and the sequence $\{(\Phi_n, S_n)\}$ also called a Markov-additive process \cite{sadowsky}.
Let $\{\tau(k)\}$ denote a sequence of stopping times, then the sequence $\{S_{\tau(k)}\}$ is a \emph{stopped random walk} \cite{anscombe60}.
\end{definition}

For a static reserve price $v$ and the Markov chain $\Phi(v)$, let $K^v_k$
denote the stopping time corresponding to the number of impression opportunities generated over $k$ days in the Markov chain $\Phi_n(v)$.
We define the expected effect of treatment $v$ over a baseline treatment $w$ as
\begin{align}\label{eq:delta}
    \Delta \triangleq \mathbb E \left[ S_{K^v_k}(v, \Phi(v)) - S_{K^w_k}(w, \Phi(w)) \right],
\end{align}
where the partial average $\bar r$ is defined in Definition~\ref{def:Phi}.
Observe that $K^v_k$ is a stopping time with respect to the Markov chain $\Phi$ as long as we also keep track of time with an additional state variable.  This fact will allow us later to apply a version of Wald's equation.

We decide to adopt treatment $v$ or $w$ based on $\Delta$ or the treatment effect $\Delta/k$ normalized by the number of time periods.
The next subsections will present an experimental method to estimate $\Delta$ along with confidence intervals for the estimator.

Although the Central Limit Theorem for Markov chains gives us a direct method to estimate the average revenue per impression $\bar r_n$, we are instead interested in the average revenue over a time period, where the number of impressions in that period is a stopping time.
In particular, the existing literature considers a random process $S_n$ with a deterministic $n$ (cf. \cite{farias2022markovian}), whereas we consider the stopped random walk $S_{K_k^v}$.

Observed that, in contrast to the potential outcome model \cite{rubin}, 
the partial sums in \eqref{eq:delta} are over different stopping times, since different reserve prices lead generally to user-sessions of different lengths and hence different Markov chains $\Phi(v)$ and $\Phi(w)$.
We consider a much more challenging inference problem where the treatment effect is measured on the time scale of days, whereas the experiment measurement is at the scale of ad impressions.
As we show in Section~\ref{sec:main}, this is possible as long as the dynamics that govern user- and advertiser- decision-making is ergodic.

A natural question is: Why do we take a measurement scale different from the treatment effect scale? The main reason is experiment cost: defining each measurement as the daily ad revenue would create long experiments with significant opportunity cost. This cost quickly becomes prohibitive when we introduce multiple treatments and control the rate of false discoveries. Our approach uses our domain knowledge of the specific dynamics of the advertising system to infer long-term effects from fewer measurements.





\subsection{Experimental design: budget-splitting}

In this section, we describe the experiment that we employ to infer the treatment effect $\tilde \Delta$ and derive confidence intervals under mild assumptions.
Our budget-splitting experiment is inspired by \cite{linkedin} and proceeds as follows.
\begin{enumerate}
    \item Assign each user in $\{1,\ldots,d\}$ into either subpopulation $A$ or $B$ with probability $q$ and $1-q$.
    
    \item Fix the number of periods $k$ for the experiment.
    
\item We assign these population to reserve price treatments $p=v$ and $p=w$,
resulting in two distinct Markov chains $\Phi^A(v)$ and $\Phi^B(w)$.

\item For each advertiser in $\{1,\ldots,m\}$, we introduce one copy for $\Phi^A(v)$ and another for $\Phi^B(w)$, each of these copies starts with a proportional fraction $q$ and $1-q$ of the budget, and every replenishment is assigned to the two copies proportionally.

\item Repeat until $k$ periods have elapsed: each impression $n$ causes a transition in $\Phi^A(v)$ if the active user is $I_n \in A$ or in $\Phi^B(w)$ otherwise. 
\end{enumerate}

At the end of the experiment, i.e., after $k$ periods, we let $V_k$ and $W_k$ denote the numbers of impressions in $\Phi^A(v)$ and in $\Phi^B(w)$, respectively.
Observe that although $A$ and $B$ have approximately the same number of users, the number of impressions $V_k$ and $W_k$ can be very different: for instance, a high reserve price $v$ may effectively prevent ads to be displayed to users in group $A$ and lead to few impressions.

At the end of the experiment, we also obtain the following impression-count and average-revenue measurements:
\begin{align}\label{eq:measurements}
    V_k, \quad W_k, \quad \bar r_{V_k}(v, \Phi^A(v)), \quad \bar r_{W_k}(w, \Phi^B(w)).
\end{align}
In turn, we can construct the following estimate for $\Delta$:
\begin{align}\label{eq:hatdelta2}
    \hat \Delta & \triangleq \frac{V_k}{q} \cdot \bar r_{V_k}(v, \Phi^A(v)) - \frac{W_k}{1-q} \cdot \bar r_{W_k}(w, \Phi^B(w)).
\end{align}
By comparing the first terms (treatment $v$) of \eqref{eq:delta} and \eqref{eq:hatdelta2}, we see that (informally) $\hat \Delta$ is a good estimator for $\Delta$ as long as we can approximate
the total number of impressions $K^v_k$ on the whole user-population Markov chain $\Phi$ by scaling by $1/q$ the corresponding number $V_k$ on the $A$-subpopulation Markov chain $\Phi^A$,
and the average revenue per impression $\bar r_{K^v_k}(v, \Phi(v))$ on $\Phi$ with the corresponding quantity $\bar r_{V_k}(v, \Phi^A(v))$  on $\Phi^A$.

\subsection{Main result: confidence interval for $\Delta$}\label{sec:main}

Our main result is a confidence interval for the long-term treatment effect $\Delta$. It requires the following standard assumptions.

\begin{assumption}[Time invariant reserve and bids]\label{as:1}
There exist functions $p$ and $Y$ such that
$p_i = p(X_i)$ and $Y_i = Y\big(X_i,p(X_i),\Omega_i\big)$ for every auction $i$.
\end{assumption}

\begin{assumption}[Ergodicity]\label{as:2}
The Markov chain $\Phi$ is aperiodic, $\varphi$-irreducible, and
Harris ergodic with invariant distribution $\pi$. Moreover, it is geometrically ergodic with $\mathbb E_\pi r^2(x) (\log^+|r(x)|) < \infty$.
\end{assumption}

\begin{remark}[Invariant distribution]
Harris ergodicity implies a unique invariant distribution, which depends on the Markov chain dynamics through the random sequences $\xi_n$, $\zeta_n$, and the functions $p$, $Y$, $h$, and $g$. 
\end{remark}

\begin{assumption}[Impressions per day converges]\label{as:K}
For a fixed reserve price $v$ (and likewise for $w$), the expected number of impressions per day converges in probability to a positive number $\eta_v \in \mathbb R$:
\begin{align*}
    K^v_k/k \to \eta_v \quad \mbox{as } k\to\infty.
\end{align*}
\end{assumption}

\begin{theorem}[Confidence interval for $\Delta$]\label{thm:main}
Suppose that the Markov chains $\Phi(v)$, $\Phi^A(v)$, $\Phi(w)$, $\Phi^B(w)$ satisfy Assumptions~\ref{as:session},~\ref{as:1},~\ref{as:2}, and~\ref{as:K}. Let $\pi(v)$, $\pi'$, $\pi(w)$, and $\pi''$ denote their respective invariant distributions.
Let $\sigma'$ and $\sigma''$ denote the asymptotic variances of $\Phi^A(v)$ and $\Phi^B(w)$, and let
\begin{align}\label{eq:56}
    \epsilon = &\max\Big(
    2\eta_v \left|\mathbb E_{\pi'} r - \mathbb E_{\pi(v)} r \right| + \frac{4}{q} z_{\alpha/2} \sigma' \sqrt{V_k},\nonumber \\
&2\eta_v \left|\mathbb E_{\pi''} r - \mathbb E_{\pi(w)} r \right| + \frac{4}{q} z_{\alpha/2} \sigma'' \sqrt{W_k}\Big),
\end{align}
where $z_{\alpha/2} = \Psi^{-1}(1 - \alpha/2)$ and $\Psi$ is the standard Normal distribution function.
Then, for every 
$\alpha>0$,
\begin{align*}
    \lim_{k\to\infty} \mathbb P(|\hat \Delta - \Delta| \leq \epsilon)
    \geq 1-2\alpha.
\end{align*}
\end{theorem}

\begin{remark}
In practice, $\Phi(v)$, $\Phi^A(v)$, $\Phi(w)$, $\Phi^B(w)$, are Markov chains generated by such large numbers ($d$ and $d/2$) of users, that we expect that the corresponding invariant distributions are close, so that $\pi(v) \approx \pi'$ and $\pi(w) \approx \pi''$. In turn, if we set $q=1/2$, then we have effectively
\begin{align*}
    &\lim_{k\to\infty} \mathbb P\left(\frac{1}{k}|\hat \Delta - \Delta| \leq \frac{8}{k} z_{\alpha/2} \max(\sigma' \sqrt{V_k},\sigma'' \sqrt{W_k}) \right)\\
    &\geq 1-2\alpha.
\end{align*}
Suppose that $V_k \approx k \eta_v / 2$ and  $W_k \approx k \eta_w / 2$ (cf. Assumption~\ref{as:K}).
If we want to detect an change of size $\epsilon_0$ in the normalized treatment effect $\Delta/k$, then the number of periods $k$ in the experiment should be set such that
\begin{align*}
    \frac{8}{\sqrt{2k}} z_{\alpha/2} \max(\sigma' \sqrt{\eta_v},\sigma'' \sqrt{\eta_w}) \approx \epsilon_0.
\end{align*}
\end{remark}

\begin{remark}[Estimating ${\sigma'}^2$ and ${\sigma''}^2$]
In practice,
the asymptotic variances ${\sigma'}^2$ and ${\sigma''}^2$ for the Markov chains $\Phi^A(v)$ and $\Phi^B(w)$ are unknown and must be estimated. This can be done, for instance, by using methods such as Regenerative Simulation and Batch Means \cite{jones04}.
\end{remark}

The proof of Theorem~\ref{thm:main} appears in the Appendix (Section~\ref{sec:proof}). It proceeds in two main steps.
First, we use a CLT for Markov chains to show a confidence interval on the long-term revenue on the Markov chains $\Phi^A(v)$ and $\Phi^B(w)$ corresponding to the two treatments.
Then, we use a Wald-like equation for Markov chains to obtain the limit of the expectation of each partial sum in $\Delta$.

\section{Related Works}

Our work falls in the growing literature on causal inference in two-sided marketplaces with interference effects (cf. \cite{johari2021experimental,mrd,farias2022markovian} and the references therein).
Instead of controlling interference between experiment measurements by a multiple-randomization technique \cite{mrd}, or with an on-policy estimator \cite{farias2022markovian},
we tackle interference by modeling the specific interactions between users, advertisers, and the publisher. In particular, we model the users' page-transition trajectories, the advertisers' bids, the publisher's reserve price, and the advertisers' budget trajectories.
Whereas \cite{johari2021experimental} carefully models the interactions in the Airbnb marketplace with a continuous-time Markov chain model and its mean-field limit, we model an online advertising publisher with a discrete-time Markov chain and its ergodic limit.


In a major departure from the existing causal inference literature, we do not consider a finite population size, but a population whose size is random and treatment-dependent. Such a setting arises naturally in problems where the population is composed of interactions between a finite number of users and a publisher (e.g., webpage or mobile application). Although the number of users is finite, the number of interactions can increase following a favorable treatment.
In this setting, we are not interested in the average treatment effect, but the total treatment effect on the random-sized population.

Our inference problem on trajectories of ad impressions is also reminiscent of causal inference on time series \cite{causalts}, using assignment paths \cite{bojinov2022design}, in Markovian settings \cite{farias2022markovian}, and in non-stationary settings \cite{Wu2022}.
However, the astronomical number of impressions in online advertising systems makes it possible to analyze those systems in a steady state limit.

We employ an experimental design similar to \cite{linkedin}, but derive confidence intervals to give guidance on how long each experiment should last. 
In contrast to other designs based on multiple randomization \cite{mrd}, Degree of Interference \cite{ohnishi2025degree}, and clustering \cite{papadogeorgou2019causal}, our design lacks however the ability to correct for arbitrary interference effects that we have not explicitly modeled.


\section{Discussions}

We have modeled the complex interactions in online advertising as a Markov chain and an associated stopped random walk. 
Another---arguably more accurate---approach is to
model the interactions as a renewal-reward process, since advertising also affect the delay between consecutive impression opportunities.
We have selected one of several possible experimental designs: an alternative is to define each entire user-session---instead of each impression---as a unit of measurement. Moreover, the budget-splitting experimental approach can be replaced by an observational approach \cite{chernozhukov2018double}.

A number of open problems jump to mind. One is to estimate the asymptotic variances $\sigma'$ and $\sigma''$ in Theorem~\ref{thm:main} with methods such as \cite{jones04}.
Another is to model nonstationary dynamics, such as new advertisers arriving, old advertisers expiring, etc.

\bibliography{ref}
\bibliographystyle{alpha}

\appendix
\onecolumn

\section{Proofs}\label{sec:proof}

\subsection{Markov chain CLT}

First, we employ the following CLT that guarantees that for large numbers of impressions in $\Phi$, the average publisher revenue converges to a limit value.


\begin{theorem}[CLT \cite{jones}]\label{thm:clt}
Suppose that $\Phi$ be a Harris ergodic Markov chain on $\mathcal X$ with invariant distribution $\pi$ and let $r:\mathcal X\to \mathbb R$ be a Borel function. 
Let $\bar r_n = \frac{1}{n} \sum_{i=1}^n r(\Phi_i)$.
If $\Phi$ is geometrically ergodic with $\mathbb E_\pi r^2(x) (\log^+|r(x)|) < \infty$, then, for any initial distribution,
\begin{align*}
    \sqrt{n} (\bar r_n - \mathbb E_\pi r) \to N(0,\sigma_r^2), \quad \mbox{as }n\to \infty.
\end{align*}
\end{theorem}

In order to apply Theorem~\ref{thm:clt}, we require that the Markov chain $\Phi$ satisfies Assumptions~\ref{as:1} and~\ref{as:2}.
Assumption~\ref{as:1} ensures that the publisher revenue function $r$ meets the requirements of the function $r$ of Theorem~\ref{thm:clt}:
\begin{align*}
    r(\Phi_i) = r\Big(p(\tilde X_i), Y\big(\tilde X_i,p(\tilde X_i),\tilde \Omega_i\big)\Big), \quad i=1,2,\ldots
\end{align*}
Observe that $\Phi$ already satisfies this assumption by \eqref{eq:p} and \eqref{eq:y}.
Assumption~\ref{as:2} ensures that the Markov chain $\Phi$ converges geometrically fast to its stationary distribution. An open problem is to use one of the conditions set out in \cite{equi} to verify that this holds for a carefully constructed pair of page $X_n$ and budget profile $\Omega_n$.


\subsection{Wald-like equation}
By applying \cite[Theorem~2.1 and~3.1]{moustakides} and observing that if the sequence $r(\Phi_n)$ converges to a limit $\mu$, then $\bar r_n$ also converges to $\mu$, we obtain the following result.

\begin{theorem}[Wald-like equation]\label{thm:wald}
Suppose that the Markov chain $\Phi(v)$ is $\varphi$-irreducible, aperiodic, and geometrically ergodic.
For a stopping time $K_k^v$ with $\mathbb E K_k^v < \infty$, we have
\begin{align*}
    \mathbb E S_{K_k^v}(v, \Phi(v)) = (\mathbb E_\pi r) (\mathbb E K_k^v) + o(\mathbb E K_k^v).
\end{align*}
\end{theorem}

\subsection{Proof of main result}
Applying Theorem~\ref{thm:clt} to $\Phi^A(v)$ and $\Phi(v)$, we obtain the following result.

\begin{lemma}[Anscombe for $\Phi(v)$ and $\Phi^A(v)$]\label{le:AC1}
Suppose that $\Phi$, $\Phi^A$, and $r$ satisfy the assumptions of Theorem~\ref{thm:clt} and Assumption~\ref{as:session}.
Then, 
\begin{align*}
    \frac{\frac{1}{K^v_k} S_{K^v_k}(v,\Phi_i) - \mathbb E_\pi r}{\sigma_r / \sqrt{K^v_k}} \xrightarrow[]{d} N(0,1),\quad
    \frac{\frac{1}{V_k} S_{V_k}(v,\Phi^A) - \mathbb E_{\pi'} r}{\sigma' / \sqrt{V_k}} \xrightarrow[]{d} N(0,1).
\end{align*}
\end{lemma}

\begin{proof}
Since $\Phi$ and $r$ satisfy the assumptions of Theorem~\ref{thm:clt}, 
let $\pi$ denote the invariant distribution, and $\sigma_r^2$ denote the asymptotic variance.
First, note that:
\begin{align*}
    U_n &= \frac{\frac{1}{n} S_n(v,\Phi_i) - \mathbb E_\pi r}{\sigma_r / \sqrt{n}},\\
    &= \frac{\frac{1}{n} \sum_{i=1}^{n} r(v,\Phi_i) - \mathbb E_\pi r}{\sigma_r / \sqrt{n}},\\
    &= \frac{\frac{1}{n} \sum_{i=1}^{n} r(Z_i) - \mathbb E_\pi r}{\sigma_r / \sqrt{n}},
\end{align*}
where $\sigma_r^2 = \textsc{var}(r(Z_1)) + 2 \sum_{i=2}^\infty \textsc{cov}_\pi(r(Z_1),r(Z_i))$.

By assumption, there exists a finite $L$ such that $\Phi_1$ and $\Phi_n$ are independent for all $n > L$, so that the infinite sum in $\sigma_r^2$ reduces to a sum of $L$ elements.
Hence, $\sigma_r^2 < \infty$.
By Theorem~\ref{thm:clt}, we have $U_n \to N(0,1)$.

Next, observe that by Assumption~\ref{as:K}, $\{K^v_k : k=1,2,\ldots\}$ is a sequence of positive integer-valued random variables that converges in probability:
\begin{align*}
    K^v_k / (k \eta_v) \to 1 \quad \mbox{as }k\to \infty.
\end{align*}

Finally, we check that $U_n$ satisfies Anscombe's Condition:
for every $\epsilon>0$ and $\eta>0$, there exist $\delta>0$ and $n_0$ such that for all $n>n_0$,
\begin{align*}
    \mathbb P\left( \max_{k:|k-n|\leq n\delta} |U_k-U_n| > \epsilon \right) < \eta.
\end{align*}
Observe that
\begin{align*}
    &\mathbb P\left( \max_{k:|k-n|\leq n\delta} |U_k-U_n| > \epsilon \right)\\
    &\approx \mathbb P\left(\max_{k=1,\ldots,2\delta} \left| \sum_{i=(1-\delta)n}^{(1-\delta)n+k} (r(v, \Phi_i) - \mathbb E_\pi r) \right| > \epsilon \sigma_r \sqrt{n} \right)\\
    &\approx \mathbb P\left(   \max_{k} \left| \sum_{\textrm{user }a} \left[ \underbrace{\sum_{i\in{\textrm{user }a\textrm{'s session}}} (r(v, \Phi_i) - \mathbb E_\pi r)}_{X_j} \right] \right| > \epsilon \sigma_r \sqrt{n} \right)\\
    &\approx \mathbb P\left(   \max_{k} \left| \sum_{{\textrm{user }a}} X_j \right| > \epsilon \sigma_r \sqrt{n} \right),
\end{align*}
where we split the sum over $2\delta$ entries into at most $2\delta/T$ terms $X_j$, each of which corresponds to a single user-session and contains at most $L$ page-transitions.
All of these terms are mutually independent by assumption,
therefore,
\begin{align*}
    \mathbb P\left( \max_{k:|k-n|\leq n\delta} |U_k-U_n| > \epsilon \right) 
    &\leq \frac{c^2}{\epsilon^2 \sigma_r^2 n} \textsc{var}\left(\sum_{j=0}^{2\delta/T-1} X_j \right)
\end{align*}
where the second inequality follows from the Kolmogorov Inequality and the fact that terms $X_j$ are independent by construction.
The claim then follows by Anscombe's Theorem \cite[Theorem~2.1]{anscombe60}:
$U_{K^v_k}\xrightarrow[]{d} N(0,1)$.

A similar argument for $\Phi^A(v)$ shows that $U_{V_k} \xrightarrow[]{d}  N(0,1)$.
\end{proof}


\begin{lemma}[Confidence interval for $\Phi^A$]\label{le:ci2}
Suppose that $\Phi^A$ and $r$ satisfy the assumptions of Lemma~\ref{le:AC1}, then for every $\alpha \in (0,1)$:
\begin{align*}
    \lim_{k\to\infty} \mathbb P\left( \left| \frac{1}{k} S_{V_k}(v,\Phi_i^A)  -  q \eta_v \mathbb E_{\pi'} r\right| \leq 2 z_{\alpha/2} \sigma' \sqrt{V_k}/k \right) \geq 1-\alpha.
\end{align*}
\end{lemma}

\begin{proof}
By Lemma~\ref{le:AC1}:
\begin{align*}
    \lim_{k\to\infty} \mathbb P\left( \left| S_{V_k}(v,\Phi^A(v)) - V_k \mathbb E_{\pi'} r \right| > z_{\alpha/2} \sigma' \sqrt{V_k} \right) = \alpha.
\end{align*}
By Assumption~\ref{as:K}, and the fact that $V_k/k = (V_k/K^v_k)(K^v_k/k) \xrightarrow[]{d} q \eta_v$,
\begin{align*}
    \lim_{k\to\infty} \mathbb P(|V_k / k - q \eta_v|>\epsilon) = 0,\\
        \lim_{k\to\infty} \mathbb P\left(\left|\frac{V_k}{k} \mathbb E_{\pi'} r - q \eta_v \mathbb E_{\pi'} r\right|>\epsilon\right) = 0.
\end{align*}

Let $A,B,C$ denote three non-negative random variables.
Observe that
\begin{align*}
 A \leq B+C &\implies \{B+C \leq \epsilon \} \subseteq \{A \leq \epsilon \}.
\end{align*}
Hence,
\begin{align}\label{eq:588}
    \mathbb P(A \leq \epsilon) &\geq P(B+C \leq \epsilon) = 1 - P(B+C> \epsilon)\nonumber \\
    &\geq 1- \mathbb P(B>\epsilon/2) - \mathbb P(C>\epsilon/2),
\end{align}

By \eqref{eq:588} and by setting $\epsilon/2 = z_{\alpha/2} \sigma' \sqrt{V_k}/k$, we obtain
\begin{align*}
    &\mathbb P\left( \left| \frac{1}{k}S_{V_k}(v,\Phi^A) - q \eta_v \mathbb E_{\pi'} r \right| \leq 2 z_{\alpha/2} \sigma' \sqrt{V_k}/k \right) \\
    &\geq 1-\mathbb P\left( \left| \frac{1}{k}S_{V_k}(v,\Phi^A) - \frac{V_k}{k} \mathbb E_{\pi'} r \right| > z_{\alpha/2} \sigma' \sqrt{V_k}/k \right) -\mathbb P\left(|\frac{V_k}{k} \mathbb E_{\pi'} r - q \eta_v \mathbb E_{\pi'} r|>\epsilon/2 \right).
\end{align*}
In turn, by taking the limit, we obtain
\begin{align*}
    \lim_{k\to\infty} \mathbb P\left( \left| \frac{1}{k} S_{V_k}(v,\Phi_i^A)  -  q \eta_v \mathbb E_{\pi'} r\right| \leq 2 z_{\alpha/2} \sigma' \sqrt{V_k}/k \right) \geq 1-\alpha.
\end{align*}
\end{proof}

Combining Theorem~\ref{thm:wald} and Lemma~\ref{le:AC1}, we obtain the following confidence interval for the estimate of publisher revenue with treatment $v$ (in Markov chain $\Phi(v)$).

\begin{theorem}[$\Phi(v)$ and $\Phi(w)$ confidence interval]\label{thm:1}
Suppose that $\Phi(v)$, $\Phi^A(v)$, $\Phi(w)$, $\Phi^B(w)$, and $r$ satisfy the assumptions of Theorem~\ref{thm:wald} 
and Lemma~\ref{le:AC1}, then for every $\alpha \in (0,1)$:
\begin{align*}
    \lim_{k\to\infty} \mathbb P\left( \left|  S_{V_k}(v,\Phi_i^A)  -  q \mathbb E S_{K^v_k}(v, \Phi(v)) \right| \leq q \eta_v \left|\mathbb E_{\pi'} r - \mathbb E_{\pi(v)} r \right| + 2 z_{\alpha/2} \sigma' \sqrt{V_k} \right) \geq 1-\alpha,\\
    \lim_{k\to\infty} \mathbb P\left( \left|  S_{W_k}(w,\Phi_i^B)  -  q \mathbb E S_{K^w_k}(w, \Phi(w)) \right| \leq q \eta_v \left|\mathbb E_{\pi''} r - \mathbb E_{\pi(w)} r \right| + 2 z_{\alpha/2} \sigma'' \sqrt{W_k} \right) \geq 1-\alpha.
\end{align*}
\end{theorem}

\begin{remark}
In practice, $\Phi$ and $\Phi^A$ are Markov chains generated by such large numbers $d$ and $d/2$ of users, that we expect that the invariant distributions $\pi$ and $\pi'$ are very close, so that $\left|\mathbb E_{\pi'} r - \mathbb E_\pi r \right|$ is negligible.
\end{remark}

\begin{proof}
By Theorem~\ref{thm:wald}, we have
\begin{align*}
    \mathbb E S_{K^v_k}(v, \tilde \Phi(v)) = (\mathbb E_\pi r) (\mathbb E K^v_k) + o(\mathbb E K^v_k)
\end{align*}
By Assumption~\ref{as:K}, we have $\mathbb E K^v_k/k \to \eta_v$, so
\begin{align}
    \lim_{k\to\infty} \left| \frac{1}{k \eta_v} \mathbb E S_{K^v_k}(v, \tilde \Phi(v)) - (\mathbb E_\pi r) \right| = 0,\\
    \lim_{k\to\infty} \left| q \frac{1}{k} \mathbb E S_{K^v_k}(v, \tilde \Phi(v)) - (q \eta_v \mathbb E_\pi r) \right| = 0
\end{align}
Recall from Lemma~\ref{le:ci2}
\begin{align*}
    \lim_{k\to\infty} \mathbb P\left( \left| \frac{1}{k} S_{V_k}(v,\Phi_i^A)  -  q \eta_v \mathbb E_{\pi'} r\right| \leq 2 z_{\alpha/2} \sigma' \sqrt{V_k}/k \right) \geq 1-\alpha.
\end{align*}
By the Triangle Inequality:
\begin{align*}
    &\left| \frac{1}{k} S_{V_k}(v,\Phi_i^A)  -  q \frac{1}{k} \mathbb E S_{K^v_k}(v, \tilde \Phi(v)) \right|\\
    &= \left| \frac{1}{k} S_{V_k}(v,\Phi_i^A)  - q \eta_v \mathbb E_{\pi'} r  +  q \eta_v \mathbb E_{\pi'} r - q \eta_v \mathbb E_\pi r  +  q \eta_v \mathbb E_\pi r - q \frac{1}{k} \mathbb E S_{K^v_k}(v, \tilde \Phi(v)) \right|\\
    &\leq \left| \frac{1}{k} S_{V_k}(v,\Phi_i^A)  - q \eta_v \mathbb E_{\pi'} r \right| + \left| q \eta_v \mathbb E_{\pi'} r - q \eta_v \mathbb E_\pi r \right| + \left| q \eta_v \mathbb E_\pi r - q \frac{1}{k} \mathbb E S_{K^v_k}(v, \tilde \Phi(v)) \right|
\end{align*}
Therefore,
\begin{align*}
    \lim_{k\to\infty} \mathbb P\left( \left| \frac{1}{k} S_{V_k}(v,\Phi_i^A)  -  q \frac{1}{k} \mathbb E S_{K^v_k}(v, \tilde \Phi(v)) \right| - \left| q \eta_v \mathbb E_{\pi'} r - q \eta_v \mathbb E_\pi r \right| \leq 2 z_{\alpha/2} \sigma' \sqrt{V_k}/k \right) \geq 1-\alpha.
\end{align*}
\end{proof}

Finally, we use Theorem~\ref{thm:1} to prove our main result as follows.

\begin{proof}[Proof of Theorem~\ref{thm:main}]
By definition, 
\begin{align*}
    \Delta &= \mathbb E \left[ S_{K^v_k}(v, \Phi(v)) - S_{K^v_k}(w, \Phi(w)) \right],\\
    \hat \Delta 
    &= \frac{1}{q} \cdot S_{V_k}(v, \Phi^A(v)) - \frac{1}{1-q} \cdot S_{W_k}(w, \Phi^B(w)).
\end{align*}
Hence,
\begin{align*}
    \hat \Delta - \Delta 
    &= \frac{1}{q} \cdot S_{V_k}(v, \Phi^A(v)) - \mathbb E \left[ S_{K^v_k}(v, \Phi(v)) \right] + \mathbb E \left[ S_{K^v_k}(w, \Phi(w)) \right] - \frac{1}{1-q} \cdot S_{W_k}(w, \Phi^B(w))
\end{align*}
or
\begin{align*}
    |\hat \Delta - \Delta| 
    &\leq \underbrace{\left|\frac{1}{q} \cdot S_{V_k}(v, \Phi^A(v)) - \mathbb E \left[ S_{K^v_k}(v, \Phi(v)) \right]\right|}_{B} + \underbrace{\left|\frac{1}{1-q} \cdot S_{W_k}(w, \Phi^B(w)) - \mathbb E \left[ S_{K^v_k}(w, \Phi(w)) \right] \right|}_{C}
\end{align*}
By \eqref{eq:588} and by setting $\epsilon$ as in \eqref{eq:56}, we have
\begin{align*}
    \mathbb P(|\hat \Delta - \Delta| \leq \epsilon) &\geq 1- \mathbb P(B>\epsilon/2) - \mathbb P(C>\epsilon/2)\\
    &\geq 1-\alpha-\alpha,
\end{align*}
where the last inequality follows from Theorem~\ref{thm:1}.
The claim follows.
\end{proof}

\end{document}